\documentclass[preprint,12pt]{elsarticle}
\usepackage{amsmath,amssymb}
\usepackage{amsthm}
\usepackage{stmaryrd}
\usepackage{enumitem}
\usepackage{tikz}
\usetikzlibrary{positioning}
\usepackage[colorlinks=true,linkcolor=blue,citecolor=blue]{hyperref}

\journal{Information Processing Letters}

\newtheorem{theorem}{Theorem}
\newtheorem{proposition}{Proposition}
\newtheorem{lemma}{Lemma}
\newtheorem{corollary}{Corollary}
\theoremstyle{definition}
\newtheorem{definition}{Definition}
\newtheorem{assumption}{Assumption}
\newtheorem{condition}{Condition}
\newtheorem{remark}{Remark}
\newtheorem{example}{Example}

\newcommand{\cl}{\mathrm{cl}}
\newcommand{\Fix}{\mathrm{Fix}}
\newcommand{\Canon}{\mathrm{Canon}}
\newcommand{\CanonCell}{\mathrm{Canon}_{\mathrm{cell}}}
\newcommand{\pow}{\wp}
\newcommand{\den}[1]{\llbracket #1 \rrbracket}
\newcommand{\eqread}{\equiv_{\mathrm{read}}}
\newcommand{\eqres}{\equiv_{\mathrm{res}}}
\newcommand{\eqreuse}{\equiv_{\mathrm{reuse}}}
\newcommand{\Rc}{\mathcal{R}_c}

\begin{document}

\begin{frontmatter}

\title{From ambiguous utterances to governed reuse classes:
canonicalization, quotient invariance, and conditional
decidability\tnoteref{t1}}
\tnotetext[t1]{This note extracts and hardens the canonicalization layer of
the working paper \emph{Certified Resolution: A Formal Theory of Governed
Answer Spaces for Enterprise AI} (Minerva CQ, 2026).}

\author{Cosimo Spera\corref{cor1}}
\ead{cosimo@minervacq.com}
\author{Ray Garcia}
\ead{ray@minervacq.com}
\cortext[cor1]{Corresponding author.}
\address{Minerva CQ (Bourbaki Intelligent Systems, Inc.), Los Gatos, CA, USA}

\begin{abstract}
Semantic caching defines answer reuse on embedding similarity: two
utterances share a stored answer when a similarity score clears a threshold,
with no notion of authorization, versioning, or of what makes two demands
the \emph{same}. This note changes the object on which reuse is defined: in
a governed domain, reuse should operate on a mathematically characterized
quotient of resolved conversational demands, not on a similarity heuristic.
Three independently defined relations on resolved utterances---reading
identity, resolution identity, and reuse identity---form a refinement chain,
strict under realized nondegeneracy conditions checkable on deployment logs;
the pipeline's outputs are invariant along the chain, and reuse identity is
exactly the kernel of the resolution map into the governed answer partition,
so the reuse quotient is the utterance-side object that partition induces,
not a relabeling of it. Reuse identity licenses the governed query key and
its certified answer space; reuse of a particular answer requires resolution
identity or an applicability certificate. The supporting layer is stated at
exactly the strength proved: exact-denotation normal forms; join
aggregation as a design operator, with closure-stable cells characterizing
no-escape; total computability of the full pipeline relative to an
untrusted proposal layer; policy admissibility for arbitrary proposers---
and provably not factual grounding or intent fidelity; and elicitation
terminating after finitely many informative replies, sound under target
consistency.
\end{abstract}


\end{frontmatter}

\section{Problem}

A governed conversational system---one whose answers must be licensed by a
corpus of policies, regulations, and standard operating procedures---receives
utterances, not questions. ``Tell me about Paris,'' ``can you do something
about this fee?,'' ``what about last month?'' are ambiguous along at least
three axes: which question is being asked, whether the domain is competent to
answer it at all, and whether enough material facts have been disclosed to
single out one answer. Retrieval pipelines dissolve the problem by ranking
documents against the raw utterance; the cost is that the system has no
representation of the question it is answering, hence no principled notion of
when two utterances ask the \emph{same} governed question---the property on
which certified answer reuse depends. Semantic caching substitutes a
similarity threshold for that notion \cite{gptcache}; this note replaces the
threshold with a quotient: the central object is the surjection
$U_{\mathrm{res}} \twoheadrightarrow \Rc$, where
$\Rc := U_{\mathrm{res}}/\!\eqreuse$ is the \emph{realized governed reuse
space}---the quotient of resolved utterances by governed reuse identity
(abstentions are terminal outcomes, held outside the quotient)---and the
theorems below say when that map is well-defined,
what it is invariant under, and what it does and does not guarantee.
Figure~\ref{fig:quotients} is the whole paper in one picture: reading
identity, resolution identity, and reuse identity form a chain of coarsening
surjections, and certified reuse is defined on the rightmost quotient.

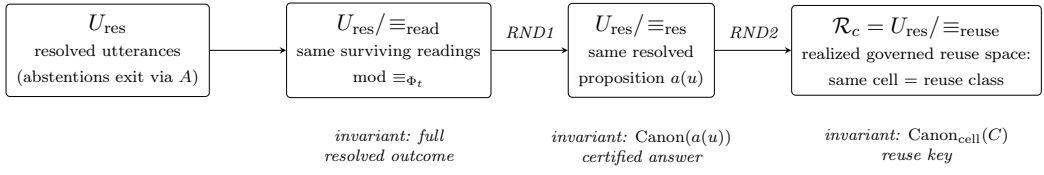
\begin{figure}[ht]
\centering
\resizebox{\textwidth}{!}{%
\begin{tikzpicture}[
  node distance=9mm and 13mm,
  q/.style={draw, rounded corners=2pt, inner sep=5pt, align=center,
            font=\small},
  lab/.style={font=\scriptsize\itshape, align=center},
  rnd/.style={font=\scriptsize\itshape},
  >=stealth]
\node[q] (U)    {$U_{\mathrm{res}}$\\[-1pt]{\scriptsize resolved
utterances}\\[-1pt]{\scriptsize (abstentions exit via $A$)}};
\node[q, right=of U]    (R) {$U_{\mathrm{res}}/\!\eqread$\\[-1pt]
  {\scriptsize same surviving readings}\\[-1pt]
  {\scriptsize mod $\equiv_{\Phi_t}$}};
\node[q, right=of R]    (S) {$U_{\mathrm{res}}/\!\eqres$\\[-1pt]
  {\scriptsize same resolved}\\[-1pt]
  {\scriptsize proposition $a(u)$}};
\node[q, right=of S]    (T) {$\Rc = U_{\mathrm{res}}/\!\eqreuse$\\[-1pt]
  {\scriptsize realized governed reuse space:}\\[-1pt]
  {\scriptsize same cell $=$ reuse class}};
\draw[->] (U) -- (R);
\draw[->] (R) -- node[rnd, above=2pt] {RND1} (S);
\draw[->] (S) -- node[rnd, above=2pt] {RND2} (T);
\node[lab, below=4mm of R] {invariant: full\\resolved outcome};
\node[lab, below=4mm of S] {invariant: $\Canon(a(u))$\\certified answer};
\node[lab, below=4mm of T] {invariant: $\CanonCell(C)$\\reuse key};
\end{tikzpicture}}
\caption{The quotient hierarchy of Theorem~\ref{thm:hier}. Each arrow is a
coarsening surjection; strictness holds under the realized nondegeneracy conditions
RND1--RND2; abstentions are terminal outcomes under the mode map $A$, not
classes of the quotient. Certified reuse is defined on the rightmost
quotient.}
\label{fig:quotients}
\end{figure}

One clarification prevents the objection that this is a relabeling. The
governed answer partition $P(c)$ lives on \emph{worlds}; the reuse quotient
lives on \emph{utterances}. The content of the construction is not the
codomain but the factorization: the resolution map
$\mathrm{res}:U_{\mathrm{res}}\to P(c)$, assigning each resolved utterance
its cell, factors through $\Rc$ with $\eqreuse$
exactly its kernel (Corollary~\ref{cor:factor})---and proving that this map
is well-defined on auditable pipeline artifacts, invariant under the finer
identities, strict under realized conditions, and computable relative to an
untrusted proposal layer is precisely what a similarity threshold cannot
offer. The quotient is moreover generally \emph{smaller} than the
partition: a cell with no realized utterance is a governed question no one
has asked, so $\Rc$ measures realized demand, not
corpus structure.

We work in the governed-answer-space framework \cite{gas}. Fix a context $c$
with finite live-world set $W_c$, an admissibility closure $\cl_c$ on
$\pow(W_c)$ (extensive, monotone, idempotent) induced by the governing corpus
$\Phi_t$ in force at time $t$, the fiber $L(c)=\Fix(\cl_c)$ of admissible
propositions, an answer partition $P(c)$ on $W_c$ in the sense of partition
semantics \cite{gro}, and a specificity threshold $\sigma(c)$ read as
partition fineness. The order convention is fixed once: entailment is
inclusion, $X\sqsubseteq Y \iff X\subseteq Y$ as sets of worlds, so a
\emph{more specific} proposition is \emph{smaller}, and
$a\sqsubseteq\sigma(c)$ reads ``$a$ is at least as specific as the
threshold''---specificity tightens downward in the lattice. All results are
relative to the snapshot $L(c,t)$; canonicity is theory-relative, and a
corpus revision re-indexes the partition and the normal forms.

Two disciplines are separated throughout. The \emph{proposal layer}---an
untrusted language model---classifies discourse function and proposes
candidate formal readings of the utterance. The \emph{verification
layer}---the governed machinery---decides admissibility, specificity, and
cell assignment. The theorems live in the verification layer; the proposal
layer enters only through an explicit computability assumption (A1), so the
decidability claims are conditional and stated as such.

\section{The pipeline, its gates, and its objects}\label{sec:gates}

Fix a finite relational signature $\Sigma$ adequate for $W_c$: each world
$w\in W_c$ is an \emph{effectively presented} finite $\Sigma$-structure and
each is distinguished by a ground description, so satisfaction
$w\models\varphi$ of closed $\Sigma$-formulas is decidable and \emph{every}
set of worlds---in particular every cell of $P(c)$ and every admissible
proposition---is defined exactly by some closed $\Sigma$-formula
(formula--cell adequacy, Condition~FA below). Propositions, cells, and the
threshold $\sigma(c)$ are represented as $n$-bit vectors over $W_c$
($n=|W_c|$), so inclusion tests cost $O(n)$. Let $\mathrm{Fm}$ be the closed $\Sigma$-formulas under a fixed
effective enumeration, and let $\prec$ be the induced
length\mbox{-}\allowbreak lexicographic total order---total, computable, and fixed once for
the deployment. Theory
equivalence, written $\varphi\equiv_{\Phi_t}\psi$, means
$\den{\varphi}=\den{\psi}$ over $W_c$; on a finite materialized fiber it is
decidable.

\begin{assumption}[A1: proposal layer]\label{a1}
There are total computable functions $C:U\to\{0,1\}$ (discourse-function
classifier: $C(u)=1$ iff $u$ is a canonical inquiry---ignorant speaker,
competent addressee, genuine gap-filling \cite{belnap}) and
$D_c:U\to\pow_{\mathrm{fin}}(\mathrm{Fm})$ (finite candidate reading set).
Nothing is assumed about their \emph{correctness}; only totality and
computability.
\end{assumption}

The pipeline, given $u$: \textbf{Gate 1} (canonicity): if $C(u)=0$, route out
of the certified pathway. \textbf{Gate 2} (admissibility): retain the readings
$\varphi\in D_c(u)$ with $\den{\varphi}\in\Fix(\cl_c)$. \textbf{Gate 3}
(specificity): retain those with $\den{\varphi}\sqsubseteq\sigma(c)$. Write
$D_c^\ast(u)$ for the survivors. \emph{If $D_c^\ast(u)=\emptyset$, the
pipeline routes to governed abstention before any aggregation}---the empty
join would be $\cl_c(\emptyset)$, and $\emptyset\subseteq C$ holds for every
cell, so cell assignment over an empty survivor set is ill-posed and is
excluded by fiat. Otherwise, if the surviving denotations lie in exactly
one cell and their aggregation (Definition~\ref{def:agg}) passes its domain
check, the pipeline \emph{resolves}; in all other cases it routes to
governed abstention (\S\ref{sec:elicit}). A reading can fail more than one gate (an
inadmissible reading may also be under-specific); the pipeline reports the
\emph{first} failing gate, yielding operationally distinguishable abstention
modes---non-canonical, inadmissible, under-specific, cell-ambiguous---and
none is repaired by guessing.

\subsection{The question object and the two normal forms}

Two conflations must be blocked at the level of definitions: a cell is not a
question, and a formula denoting a subset of a proposition is not a normal
form for it.

\begin{definition}[Cell-indexed governed query]\label{def:query}
For a cell $C\in P(c)$, the governed query $Q_C$ is the query whose
admissible responses are exactly
\[
\mathrm{Ans}(Q_C) \;=\;
\bigl\{\,a\in\Fix(\cl_c)\;:\;a\neq\emptyset,\ a\subseteq C,\
a\sqsubseteq\sigma(c)\,\bigr\},
\]
the normatively answerable propositions whose denotation lies in $C$. We call
$C$ the \emph{positive resolution region} of $Q_C$.
\end{definition}

\begin{remark}[Relation to partition semantics]\label{rem:partition}
Under Groenendijk--Stokhof partition semantics a question denotes a partition
of logical space; the classical object nearest to $Q_C$ is the bipolar
partition $\{C,\ W_c\setminus C\}$ \cite{gro}. $Q_C$ is deliberately the
\emph{restricted, operational} notion: a governed deployment certifies only
positive resolutions inside $C$---the complement is not one answer but the
union of the other cells and the abstention region, each governed on its own
terms. We therefore do not claim that $Q_C$ \emph{is} a question in the
classical sense; it is the cell-indexed query object that governed reuse is
keyed on, and every statement below about ``the governed question of $C$''
abbreviates $Q_C$.
\end{remark}

\begin{definition}[Cell normal form; the reuse key]\label{def:canoncell}
For $C\in P(c)$, let $F(C)=\{\varphi\in\mathrm{Fm}:\den{\varphi}=C\}$ and
\[
\CanonCell(C) \;=\;
\min\nolimits_{\prec}\bigl(\arg\min_{\varphi\in F(C)}|\varphi|\bigr).
\]
$\CanonCell(C)$ is the canonical presentation of the positive resolution
region of $Q_C$, and the reuse class of an utterance is keyed by
$\CanonCell$ of its resolved cell, not by any particular reading.
\end{definition}

\begin{definition}[Proposition normal form, exact denotation]\label{def:canon}
Let
\begin{align*}
\mathrm{Dom}(\Canon_c) \;=\;
\bigl\{\,a\in\Fix(\cl_c)\;:\;{}&a\neq\emptyset,\ a\sqsubseteq\sigma(c),\\
&\exists\,C\in P(c)\ \text{with}\ a\subseteq C\,\bigr\}.
\end{align*}
Nonemptiness is essential and does real work: since $P(c)$ is a partition, a
\emph{nonempty} $a$ contained in a cell is contained in \emph{exactly one}
cell, so ``the cell of $a$'' is well-defined on the domain---whereas
$\emptyset\subseteq C$ holds for every cell, and admitting $\emptyset$
would make cell assignment ill-posed. The exclusion also closes a re-entry
route for contradiction: the $\bot$ that exactness bars from the normal form
(Remark~\ref{rem:exact}) is equally barred from the resolution domain.
For $a\in\mathrm{Dom}(\Canon_c)$, let
$F(a)=\{\varphi\in\mathrm{Fm}:\den{\varphi}=a\}$---\emph{exact} denotation,
not containment---and $\Canon(a)=\min_{\prec}(\arg\min_{\varphi\in F(a)}
|\varphi|)$.
\end{definition}

\begin{remark}[Why exactness is forced]\label{rem:exact}
With containment ($\den{\varphi}\subseteq a$) in place of equality, the
normal form collapses: $\den{\bot}=\emptyset\subseteq a$ for every $a$, so
any contradiction of minimal length would be the ``normal form'' of every
proposition, and equal normal forms would carry no information about the
propositions. Exactness restores the defining property: $\Canon(a)=
\Canon(a')$ iff $a=a'$, since equal formulas have equal denotations and
$\den{\Canon(a)}=a$ by construction. A containment-based variant
$F_{\mathrm{wit}}(a)=\{\varphi:\emptyset\neq\den{\varphi}\subseteq a\}$
defines a \emph{witness} normal form---useful for evidence, but not a
proposition identifier---and is not used here.
\end{remark}

\begin{lemma}[Well-definedness and computability]\label{lem:wd}
On a finite fiber, $\CanonCell$ is a single-valued total computable function
on $P(c)$, and $\Canon$ is a single-valued total computable function on
$\mathrm{Dom}(\Canon_c)$ that is moreover injective:
$\Canon(a)=\Canon(a')$ implies $a=a'$.
\end{lemma}

\begin{proof}
Non-emptiness of $F(C)$ and $F(a)$ is Condition~FA: every world set is
defined exactly by some closed formula (a disjunction of ground world
descriptions). The set of minimum-length candidates is finite and nonempty;
the fixed total computable order $\prec$ therefore selects a unique least
element. Computability: enumerate formulas in $\prec$-order; each test
$\den{\varphi}=a$ is a finite denotation check against the materialized
fiber; the first formula passing is the normal form. Injectivity is
Remark~\ref{rem:exact}.
\end{proof}

\subsection{Aggregation, and the compatibility conditions}

\begin{definition}[Resolution aggregation---a design operator]\label{def:agg}
Given a \emph{nonempty} $D_c^\ast(u)$ (the empty case having been routed to
abstention upstream), the \emph{resolved proposition} is the deterministic
join
\[
a(u) \;=\; \bigvee_{\varphi\in D_c^\ast(u)} \den{\varphi}
\;=\; \cl_c\Bigl(\bigcup_{\varphi\in D_c^\ast(u)}\den{\varphi}\Bigr),
\]
and the pipeline resolves iff $a(u)\in\mathrm{Dom}(\Canon_c)$: the join
must itself be nonempty, cell-contained, and clear $\sigma(c)$, re-checked
\emph{after} aggregation. In particular $a(u)=\emptyset$ (e.g.\ when
$\cl_c(\emptyset)=\emptyset$ and only empty-denotation readings survive)
fails the domain check and routes to governed abstention.
\end{definition}

\begin{remark}[Status of the aggregation]\label{rem:agg}
Definition~\ref{def:agg} is a \emph{conservative design operator, not a
semantic theorem}. In formal semantics, ambiguity is standardly preserved as
a \emph{set} of alternatives $\{a_1,\dots,a_k\}$; collapsing it destroys
intensional distinctions among readings. We collapse deliberately, for two
governed-system reasons: the emitted answer must be a single certifiable
proposition, and the choice must be deterministic and auditable. The join is
the conservative collapse in the answer direction: it is the least admissible
upper bound of the surviving denotations---their coarsest common admissible
cover in the specificity interpretation---and claims no specificity that any
reading would withhold; and the intensional distinctions are not lost but
relocated: the full surviving reading set is logged as the auditable artifact
(\S\ref{sec:invariance}). Two consequences of collapsing are then handled
explicitly rather than assumed away. If the join falls above $\sigma(c)$,
residual ambiguity has degraded specificity below the material bar and the
pipeline elicits instead of answering. And because $\cl_c$ is extensive, the
join may contain worlds contributed by \emph{closure}, not by any reading;
whether that can push $a(u)$ across a cell boundary is exactly the
compatibility question below.
\end{remark}

\begin{condition}[FA: formula--cell adequacy]\label{cond:fa}
Every set of worlds over the finite base $W_c$ is defined exactly by some
closed $\Sigma$-formula. (Assumed throughout; discharged by including ground
world descriptions in $\Sigma$.)
\end{condition}

\begin{condition}[CP: closure--partition compatibility]\label{cond:cp}
Every cell of $P(c)$ is admissible: $C\in\Fix(\cl_c)$ for all $C\in P(c)$.
\end{condition}

\begin{lemma}[CP characterizes no-escape]\label{lem:cp}
For a cell $C\in P(c)$, the following are equivalent: (i) $\cl_c(S)\subseteq
C$ for every $S\subseteq C$; (ii) $\cl_c(C)=C$. Hence under CP, if every
surviving denotation lies in one cell $C$, then $a(u)\subseteq C$:
closure-induced cross-cell escape cannot occur, and the post-aggregation
domain check can fail only at the $\sigma$-threshold. Without CP, escape is
possible, and the domain check of Definition~\ref{def:agg} detects it and
routes to elicitation.
\end{lemma}

\begin{proof}
(ii)$\Rightarrow$(i): monotonicity gives $\cl_c(S)\subseteq\cl_c(C)=C$.
(i)$\Rightarrow$(ii): take $S=C$ and use extensivity, $C\subseteq\cl_c(C)
\subseteq C$. The consequence is (i) applied to
$S=\bigcup_{\varphi\in D_c^\ast(u)}\den{\varphi}\subseteq C$.
\end{proof}

\begin{remark}[CP as a normative design principle, with a caveat]\label{rem:cp}
CP has a governance reading that makes it more than a technical convenience:
\emph{the positive resolution region of a governed question should itself be
a licensed proposition}---equivalently, a governed question cell should be
closed under admissible interpretation aggregation. But enforcement is not
free: adding a cell to the generating set of $\cl_c$ licenses the
proposition ``some admissible answer in $C$ holds,'' which may be broader
than a regulated corpus permits. Deployments may therefore enforce CP where
cells are themselves licensable propositions; otherwise CP simply fails on
those cells, and Lemma~\ref{lem:cp} says the failure is \emph{detected} at
the domain check and routed to elicitation, never silently absorbed into an
answer. CP is a diagnostic and an option, not a universal requirement.
\end{remark}

\begin{condition}[TC: threshold re-check]\label{cond:tc}
The $\sigma$-comparison is applied to $a(u)$ after aggregation, not only to
readings individually (built into Definition~\ref{def:agg}); per-reading
specificity does not imply joint specificity, since the join is coarser than
each reading.
\end{condition}

\begin{condition}[PC: proof-producing closure]\label{cond:pc}
The materialization of $\cl_c$ records, for every $a\in\Fix(\cl_c)$, a
derivation witness $\mathrm{Cert}_c(a)$---on a ground Horn fiber, the
how-provenance of the forward-chaining derivation of $a$'s generators. The
map $\mathrm{Cert}_c$ is total on $\Fix(\cl_c)$ and computable from the
materialization, and each witness is bound to the fiber data
$(\Sigma,\Phi_t,c,t)$. For the proof-producing Horn materializations
considered here, PC can be implemented as a bookkeeping discipline by
recording compact derivation provenance (a derivation DAG, not a fully
expanded proof object) during forward chaining, without changing the
asymptotic materialization bound; for closures given in other forms, PC is
an assumption on the implementation, which is why it is stated as a
condition; it is what makes ``carries a certificate'' a derived
property rather than an assertion (Proposition~\ref{prop:auth}).
\end{condition}

The stored artifact is per cell: the pair $\langle$natural-language normal
form of $\CanonCell(C)$, first-order witness$\rangle$; the cell normal form
is
the reuse key, and $\Canon(a(u))$ is the certified representative of the
emitted answer within it.

\begin{example}[Running example: fee waiver vs.\ refund]\label{ex:run}
Context $c$: consumer billing, jurisdiction requiring a documented hardship
finding before any late-fee waiver. Ground atoms include
$\mathit{feeAssessed}$, $\mathit{hardshipDoc}$, $\mathit{cancelWindowOpen}$;
$P(c)$ contains (among others) the cells $C_{\mathrm{waive}}$ (late-fee
waiver eligibility) and $C_{\mathrm{refund}}$ (refund eligibility), both
admissible (CP holds). Utterance $u=$ ``can you do something about this
fee?'' The proposal layer returns $C(u)=1$ and
$D_c(u)=\{\varphi_1,\varphi_2,\varphi_3\}$: $\varphi_1$ (waiver eligibility
given $\mathit{hardshipDoc}$), $\varphi_2$ (a syntactic variant with
$\varphi_2\equiv_{\Phi_t}\varphi_1$), and $\varphi_3$ (refund eligibility).
Gates 2--3 retain all three ($\varphi_3$ is admissible---refunds are a
licensed topic). The surviving denotations meet \emph{two} cells, so $u$ is
cell-ambiguous and routes to elicitation (\S\ref{sec:elicit}): the
cell-decided predicate $\mathit{cancelWindowOpen}\in\mathcal{E}_c$ is asked;
the reply ``the cancellation window has closed'' contradicts
$C_{\mathrm{refund}}$. Now $D_c^\ast(u)=\{\varphi_1,\varphi_2\}$,
$a(u)=\den{\varphi_1}\vee\den{\varphi_2}=\den{\varphi_1}\in
\mathrm{Dom}(\Canon_c)$ (by CP the join stays in $C_{\mathrm{waive}}$), the
resolved cell is $C_{\mathrm{waive}}$, and the emitted governed question is
$Q_{C_{\mathrm{waive}}}$ presented by $\CanonCell(C_{\mathrm{waive}})$---in
normal-form English, ``under $\Phi_t$, is this customer eligible for a
late-fee waiver?''---with certified answer $\Canon(a(u))$ and reuse keyed by
the cell. A later utterance ``does the hardship rule let you drop the late
charge?'' whose surviving readings are theory-equivalent to $\varphi_1$ is
reading-equivalent to the post-elicitation $u$ and, by
Theorem~\ref{thm:hier}, hits the same reuse class with no regeneration.
\end{example}

\section{The equivalence hierarchy}\label{sec:invariance}

The relation that keys reuse must be defined \emph{without reference to} the
normal forms, else invariance is a quotient triviality---and it is not one
relation but three, at increasing coarseness. All are auditable artifacts of
the pipeline, not properties of strings. One structural choice precedes
them: abstentions are terminal outcomes, not members of reuse classes.

\begin{definition}[Resolution split and mode map]\label{def:split}
Let $U_1=\{u\in U:C(u)=1\}$ and let $U_{\mathrm{res}}\subseteq U_1$ be the
utterances the pipeline resolves. The \emph{mode map}
$A:U_1\setminus U_{\mathrm{res}}\to M$ assigns each non-resolving utterance
its abstention mode: inadmissible, under-specific, empty survivor set,
cell-ambiguous, or aggregation failure (the post-aggregation domain check).
Non-canonical is deliberately \emph{not} in $M$: Gate-1 rejects have
$C(u)=0$ and never enter $U_1$, so the mode map's domain begins after
canonicity.
\end{definition}

\begin{definition}[Three equivalences]\label{def:equiv}
All three relations are defined on $U_{\mathrm{res}}$:
\begin{enumerate}[label=(\roman*),nosep]
\item $u \eqread u'$ (\emph{same surviving readings}) iff
$D_c^\ast(u)/\!\equiv_{\Phi_t}\;=\;D_c^\ast(u')/\!\equiv_{\Phi_t}$;
\item $u \eqres u'$ (\emph{same resolved proposition}) iff $a(u)=a(u')$;
\item $u \eqreuse u'$ (\emph{same governed question}) iff $u$ and $u'$
resolve in the same cell of $P(c)$.
\end{enumerate}
The central object is $\Rc := U_{\mathrm{res}}/\!\eqreuse$, the
\emph{realized governed reuse space}: the partition of resolved
conversational demands into governed reuse classes.
\end{definition}

\begin{remark}[Why abstention is not an equivalence]\label{rem:mode}
An alternative extends $\eqres$ and $\eqreuse$ to all of $U_1$ by declaring
utterances with the same abstention mode equivalent. We reject it: two
utterly unrelated queries---one about a mortgage, one about a battery
warranty---would become resolution-equivalent merely by both being
under-specific, and the reuse quotient would then contain, alongside the
governed reuse classes, giant classes of unrelated abstentions with no
shared certified content. Modes classify \emph{failures of} resolution, not
resolutions; they are the codomain of $A$, not cells of the quotient.
Surviving-reading identity does extend meaningfully to non-resolving
utterances in one respect---it preserves \emph{membership} in
$U_{\mathrm{res}}$ (Remark~\ref{rem:trace})---but it cannot classify
abstentions, which is the second reason the relations live on
$U_{\mathrm{res}}$.
\end{remark}

\begin{remark}[Why $D_c^\ast$ cannot see the mode]\label{rem:trace}
Extending the relations to $U_1$ by surviving readings alone would make a
false claim provable-looking: $D_c^\ast(u)$ is the survivor set \emph{after}
Gates 2--3 and has forgotten why readings were eliminated. Concretely, take
$D_c(u)=\{\varphi\}$ with $\den{\varphi}\notin\Fix(\cl_c)$ (Gate-2
failure, mode inadmissible) and $D_c(u')=\{\psi\}$ with
$\den{\psi}\in\Fix(\cl_c)$ but $\den{\psi}\not\sqsubseteq\sigma(c)$
(Gate-3 failure, mode under-specific): then
$D_c^\ast(u)=D_c^\ast(u')=\emptyset$, the utterances are
surviving-reading identical, and $A(u)\neq A(u')$. What \emph{is} true on
$U_1$ is the membership claim: if $D_c^\ast(u)/\!\equiv_{\Phi_t}=
D_c^\ast(u')/\!\equiv_{\Phi_t}$ then every decision from the survivor set
onward is denotational, so $u$ and $u'$ agree on whether they resolve---but
not, when both abstain with empty survivors, on which earlier gate failed. A
finer \emph{trace equivalence}, comparing the per-gate quotients
$D_c^{(0)}/\!\equiv_{\Phi_t}$, $D_c^{(2)}/\!\equiv_{\Phi_t}$,
$D_c^{(3)}/\!\equiv_{\Phi_t}$, would preserve modes as well; we note it as
a refinement and do not develop it here.
\end{remark}

\begin{theorem}[Refinement hierarchy and invariance]\label{thm:hier}
On $U_{\mathrm{res}}$,
${\eqread}\subseteq{\eqres}\subseteq{\eqreuse}$ always. Because the
equivalences are relations on utterances, strictness
depends on which reading sets the proposal layer actually realizes, not only
on the fiber; it holds under the following \emph{realized} nondegeneracy
conditions, and may collapse when the realized image of $D_c$ is too
poor:
\begin{enumerate}[label=(RND\arabic*),nosep]
\item there exist $u,u'\in U$ with $C(u)=C(u')=1$, both resolving, such that
$D_c^\ast(u)/\!\equiv_{\Phi_t}\neq D_c^\ast(u')/\!\equiv_{\Phi_t}$ and
$a(u)=a(u')$ \ (then ${\eqread}\subsetneq{\eqres}$);
\item there exist $u,u'\in U$ both resolving in the same cell with
$a(u)\neq a(u')$ \ (then ${\eqres}\subsetneq{\eqreuse}$).
\end{enumerate}
Moreover, on $U_{\mathrm{res}}$, the entire resolved outcome---(cell,
$a(u)$, $\CanonCell$, $\Canon(a(u))$)---is invariant under $\eqread$; the
resolved proposition and its certified representative are invariant under
$\eqres$; and the governed question, the cell normal form, and the reuse
class are invariant under $\eqreuse$. (Abstention modes are \emph{not}
claimed invariant under surviving-reading identity;
Remark~\ref{rem:trace}.) Certified reuse of the query key is defined on $\Rc$.
\end{theorem}

\begin{proof}
\emph{Inclusions.} On $U_{\mathrm{res}}$, every pipeline decision from the
survivor set onward is a function of the set of denotations
$\{\den{\varphi}:\varphi\in D_c^\ast(u)\}$: cell incidence, the join
$a(u)$, and the post-aggregation domain check are all denotational.
Theory-equivalent readings have equal denotations, so $\eqread$ forces
equal denotation sets, hence $a(u)=a(u')$, giving
${\eqread}\subseteq{\eqres}$. If $a(u)=a(u')$, then since
$a(u)\in\mathrm{Dom}(\Canon_c)$ is nonempty, the cell containing the
common proposition is unique (Definition~\ref{def:canon}) and shared,
giving ${\eqres}\subseteq{\eqreuse}$.

\emph{Strictness.} The witnesses are supplied by the conditions themselves:
under RND1 the pair $(u,u')$ is $\eqres$- but not $\eqread$-related; under
RND2 it is $\eqreuse$- but not $\eqres$-related. Nothing further is needed,
which is the point of quantifying over realized utterances.

\emph{Invariances.} Under $\eqread$, all decisions coincide as above. Under
$\eqres$, $\Canon(a(u))=\Canon(a(u'))$ by single-valuedness
(Lemma~\ref{lem:wd}) applied to the common proposition, computed against the
same fixed $(\Sigma,\prec,\Phi_t)$. Under $\eqreuse$, $\CanonCell$ of the
common cell coincides, and the reuse key is by definition a function of the
cell.
\end{proof}

\begin{remark}[Fiber-level vs.\ realized nondegeneracy]\label{rem:rnd}
The corpus-side conditions one might state instead---(ND1) there exist
admissible $a_1\neq a_2$ with $a_1\vee a_2\in\mathrm{Dom}(\Canon_c)$; (ND2)
some cell contains two distinct propositions of $\mathrm{Dom}(\Canon_c)$---
are \emph{necessary} for RND1/RND2 respectively (a witness pair realizes the
corresponding fiber configuration) but not sufficient: Condition~FA supplies
formulas, not utterances, and a fixed proposal layer
$D_c:U\to\pow_{\mathrm{fin}}(\mathrm{Fm})$ may never output the surviving
sets a fiber configuration would require. Strictness on $U_{\mathrm{res}}$ is therefore a
joint property of corpus and proposal layer: ND1/ND2 say the fiber is rich
enough to \emph{permit} separation, RND1/RND2 say the deployment
\emph{exhibits} it. Empirically, RND1/RND2 are checkable on logs, since
reading sets, joins, and cells are all recorded pipeline artifacts.
\end{remark}

\begin{corollary}[Quotient observation]\label{cor:quot}
By injectivity of $\Canon$ (Lemma~\ref{lem:wd}), $\Canon(a(u))=
\Canon(a(u'))$ implies $a(u)=a(u')$, hence $u\eqres u'$ and $u\eqreuse u'$:
the resolution and reuse assignments factor through the proposition normal
form. (This direction was invalid under containment-based $F(a)$;
Remark~\ref{rem:exact}.)
\end{corollary}

\begin{corollary}[Factorization; the quotient is not a relabeling of
$P(c)$]\label{cor:factor}
Let $\mathrm{res}:U_{\mathrm{res}}\to P(c)$ assign each resolved utterance
the cell of $a(u)$ (well-defined by nonemptiness,
Definition~\ref{def:canon}). Then $\eqreuse$ is exactly the kernel of
$\mathrm{res}$, so $\mathrm{res}$ factors as
\[
U_{\mathrm{res}}
\;\twoheadrightarrow\;
\Rc
\;\hookrightarrow\;
P(c),
\]
with the induced map injective, and surjective iff every cell is realized by
some resolved utterance. Hence $\Rc$ is the utterance-side object the
partition induces: it coincides with $P(c)$ only when the deployment
realizes every governed question, and in general it measures \emph{realized
demand}---whence its name--- over the corpus structure---the object on which reuse economics
(hit rates over a query stream) is actually defined.
\end{corollary}

\begin{proof}
By Definition~\ref{def:equiv}(iii), $u\eqreuse u'$ iff
$\mathrm{res}(u)=\mathrm{res}(u')$, which is the definition of the kernel;
the factorization and injectivity of the induced map are the universal
property of quotients by a kernel, and surjectivity onto $P(c)$ is by
construction equivalent to every cell having a preimage.
\end{proof}

\begin{remark}[What $\eqreuse$ licenses]\label{rem:whatreused}
The two coarser relations name two different reusable objects, and the
distinction should be quoted precisely. $\eqreuse$ is \emph{query-key}
identity: it licenses reuse of the governed query key
$\CanonCell(C)$ and, with it, the certified answer space
$\mathrm{Ans}(Q_C)$---the CAS lookup namespace. It does \emph{not} by
itself license reuse of a particular stored answer: RND2 exhibits
$a(u)\neq a(u')$ inside one cell, with distinct certified representatives
$\Canon(a(u))\neq\Canon(a(u'))$. Reuse of a particular answer additionally
requires resolution equivalence $\eqres$---the same exact certified
proposition---or an independent applicability certificate that the stored
answer's proposition contains, at the required specificity, the one resolved
now. In slogan form: $\eqreuse$ reuses the \emph{question}; $\eqres$
reuses the \emph{answer}.
\end{remark}

\begin{remark}[The answer space of a governed question; where uniqueness
lives]\label{rem:ansspace}
A natural objection runs: since a certified answer carries a unique
certificate, should not each governed question determine a unique certified
answer---and does the question-to-answer multiplicity not break the
mapping? The objects resolve it. First, question and canonical form are one
object, not two: by Lemma~\ref{lem:wd} and exactness, $C\mapsto\CanonCell(C)$
is a bijection between cells and canonical forms, so canonicalization is
how an utterance \emph{reaches} the governed question, not a second map
applied after it. Second, every uniqueness holds at the correct arrow. The
resolved pipeline is the chain
\[
u \;\longmapsto\; a(u) \;\longmapsto\; C \;\longmapsto\; \CanonCell(C),
\qquad
a(u) \;\longmapsto\; \bigl(\Canon(a(u)),\ \mathrm{Cert}_c(a(u))\bigr),
\]
and each displayed arrow is a function: the join is deterministic
(Definition~\ref{def:agg}); the cell of $a(u)$ is unique because
$a(u)\neq\emptyset$ and $P(c)$ partitions (Definition~\ref{def:canon});
cell to canonical form is the bijection above; and, per fiber snapshot
$L(c,t)$, the certificate is a function of the answer (Condition~PC),
re-minted on theory change. The one direction that is \emph{not} a
function---question to answer---is not supposed to be: $\mathrm{Ans}(Q_C)$
is a $\sqsubseteq$-ordered family whose members are the same resolution at
different material specificity, RND2 is precisely its realized
multiplicity, and \emph{which} member an interaction receives is fixed by
the accumulated case facts through $a(u)$, not by the question alone. The
store mirrors this: the key is the question, the entries are the answers,
each with its own certificate---a certificate family per question, never
two certificates per answer. The one-answer-per-question intuition is the
fiber condition $\neg$ND2 (every cell contains a single proposition of
$\mathrm{Dom}(\Canon_c)$), under which $\eqres$ and $\eqreuse$ coincide;
it cannot in general be enforced by refining the partition, since distinct
certified answers in a cell typically nest and nested propositions admit no
disjoint separation, and it forfeits the applicability-reuse tier. A
related question is open: when $\mathrm{Ans}(Q_C)$ has a
$\sqsubseteq$-greatest element---the coarsest certified proposition of the
cell at threshold---that element is a canonical \emph{default answer},
recovering a distinguished question-to-answer section; characterizing the
fibers on which such defaults exist is left to future work.
\end{remark}

\begin{remark}[Scope of the invariance]
Theorem~\ref{thm:hier} is invariance relative to the reading sets: the
governed machinery cannot distinguish utterances whose admissible readings
agree modulo the theory, and reuses across utterances that agree at the
coarser resolution or cell level even where readings differ---the hierarchy
is exactly the statement that reuse equivalence is a strict coarsening of
reading-level equivalence. What no theorem here says is that two
utterances a human would call paraphrases always receive equal reading sets;
that is a property of the proposal layer. The theorems' force is that the
residual risk is localized entirely in $D_c(u)$, an auditable, loggable
artifact, rather than diffused through the pipeline.
\end{remark}

\section{Decidability and complexity, by regime}\label{sec:complexity}

\begin{theorem}[Conditional decidability]\label{thm:dec}
Under Assumption~A1 and a finite materialized fiber $(L(c),P(c),\sigma(c))$,
the full pipeline map
\[
u \;\longmapsto\;
\{\textsf{non-canonical route}\}\;\uplus\;M\;\uplus\;
\{\,Q_C : C\in P(c)\,\}
\]
is total computable on all of $U$: Gate-1 rejects ($C(u)=0$) receive the
routing outcome, utterances in $U_1\setminus U_{\mathrm{res}}$ receive
their mode $A(u)\in M$, and resolved utterances receive their governed
question. All corpus-side operations---fixpoint membership, the
$\sigma$-comparison, cell incidence, the join aggregation $a(u)$, the CP
check, and both normal forms---are decidable with no assumption on the
proposal layer.
\end{theorem}

\begin{proof}
$C$ is total computable by A1, so the Gate-1 routing branch is decided for
every $u\in U$; on $U_1$, $D_c$ is total computable by A1, and each
subsequent step is a finite check or computation against the
materialization (Gates 2--3, the mode assignment from the gate trace, cell
incidence, the join $a(u)$ as one closure application, the post-aggregation
domain check, and $\cl_c(C)=C$ per cell for CP) or is computable by
Lemma~\ref{lem:wd} ($\CanonCell$, $\Canon$). The three branches of the
codomain are exhaustive and mutually exclusive by construction of the
pipeline, and the composition of totally computable steps over finite data
is totally computable.
\end{proof}

\begin{remark}[Status of Theorem~\ref{thm:dec}]
Given A1 and a finite materialized fiber, the theorem is deliberately
immediate: its content lies in the \emph{placement} of the assumptions---all
non-computability is quarantined in the proposal layer, and everything
corpus-side is unconditionally decidable---not in the composition argument.
The load-bearing results of the note are the hierarchy
(Theorem~\ref{thm:hier}) and the compatibility lemma (Lemma~\ref{lem:cp});
the decidability statement records that the quotient they describe is
effectively computable in deployment.
\end{remark}

\begin{proposition}[Complexity, three regimes]\label{prop:cx}
Let $n=|W_c|$, let $P_c$ be the program generating $\cl_c$, and let $m$ be
the number of literal occurrences in the grounding of $P_c$.
\begin{enumerate}[label=(\roman*),nosep]
\item \emph{Ground Horn tier.} If $P_c$ is a ground (propositional) Horn
program, the closure of a fact set is computable in time $O(m)$ by
unit-propagation-style indexed forward chaining, in the manner of the
linear-time Horn satisfiability algorithms of Dowling and Gallier
\cite{dowling} (the citation supports the propagation technique; the closure
statement is the standard consequence under indexed representations). With
$L(c)$, $P(c)$ materialized offline and propositions, cells, and $\sigma(c)$
held as $n$-bit vectors, Gate~2 is a lookup and Gate~3 a bit-vector
inclusion test in $O(n)$. Cell assignment is stated precisely: containment
against a \emph{given} candidate cell costs $O(n)$; \emph{locating} the
containing cell na\"ively costs $O(|P(c)|\,n)$ over the partition, reducible
to $O(n)$ amortized under a precomputed world-to-cell index (any world of
$a(u)$ names the candidate, leaving one $O(n)$ containment check). Each
pipeline operation is polynomial in $m+n$. This is the regime of the
production architecture \cite{gas,capdatalog}.
\item \emph{Fixed-program (data) complexity.} For non-ground Datalog with
$P_c$ fixed, data complexity is measured in the size $N$ of the extensional
input structure (the EDB), in which Datalog is \textsc{PTime}-complete
\cite{abiteboul}. In this architecture $W_c$ is \emph{extensionally
materialized}---the bit-vector representation over the explicit world set
is the input structure---so $n\leq N$ and materialization and all gate
checks are polynomial in $N$. Without that materialization assumption,
polynomiality in the number of semantic worlds is not implied and is not
claimed.
\item \emph{Combined complexity.} With both program and data varying, Datalog
is \textsc{ExpTime}-complete \cite{abiteboul}; no polynomial claim is made in
this regime, and none is needed: corpus compilation fixes $P_c$ offline, so
query-time operation is governed by (i)--(ii).
\end{enumerate}
Both normal forms are computable (Lemma~\ref{lem:wd}) but their na\"ive
enumeration is not polynomial; in deployment $\CanonCell(C)$ is computed
offline once per cell and cached, and $a(u)$ costs one closure application,
so the query-time cost of canonicalization is the cell assignment of
(i)--(ii), not formula search.
\end{proposition}

\section{Three orthogonal safety properties}\label{sec:safety}

The closure proves less than ``safety'' and the paper must say exactly what.
Three properties come apart.

\begin{definition}\label{def:safety}
A pipeline run has:
\begin{enumerate}[label=(\alph*),nosep]
\item \emph{policy admissibility} if every emitted proposition lies in
$\Fix(\cl_c)$, clears $\sigma(c)$, is cell-contained, and carries a
certificate bound to its fiber---``can we say it?'';
\item \emph{factual grounding} if every premise on which the emitted
proposition's derivation rests is a verified fact of the live case---``is it
supported?'';
\item \emph{intent fidelity} if the emitted answer resolves the question the
speaker in fact intended---``is it what the user asked?''.
\end{enumerate}
\end{definition}

\begin{proposition}[Policy admissibility for arbitrary proposers]
\label{prop:auth}
For every proposal layer satisfying A1---including an adversarial one---every
run of the pipeline has policy admissibility: the emitted proposition is
precisely $a(u)$ (equivalently its certified representative $\Canon(a(u))$,
whose denotation is $a(u)$ by exactness), which has passed the fixpoint,
threshold, and cell-containment checks \emph{after} aggregation, computed by
the verification layer against the materialized fiber, independently of how
the proposals were produced; and by Condition~PC the emission carries
$\mathrm{Cert}_c(a(u))$, a derivation witness bound to the fiber data
$(\Sigma,\Phi_t,c,t)$, so certificate possession is derived, not assumed.
\end{proposition}

\begin{remark}[What the closure does not guarantee]\label{rem:fidelity}
Neither factual grounding nor intent fidelity follows from
Proposition~\ref{prop:auth}, and the two fail differently. \emph{Grounding:}
an adversarial (or merely wrong) proposer can supply a reading whose
denotation is perfectly admissible---``customer has documented hardship''---
while the fact is false of this customer; the closure checks licensure, not
evidence, so the pipeline would emit an authorized but factually unsupported
answer. Grounding is enforced by a separate evidence certificate layer---the
verified-fact discipline whose records the Compliance Certificate Token's
grounding field carries---not by the canonicalization theorems of this note,
which make no claim about it. \emph{Fidelity:} a mis-canonicalized reading can pass every gate
and yield a fully certified, fully grounded answer to a \emph{different}
admissible question---cancellation rights mapped to refund eligibility---
which in a regulated setting is a material failure. The architecture's
guarantees must be quoted precisely: the closure bounds \emph{what may be
said}; it does not verify \emph{what is true of the case} or \emph{what was
asked}. The mitigations are structural, not theorematic:
Theorem~\ref{thm:hier} localizes fidelity risk in the reading set $D_c(u)$;
the natural-language normal form of $\CanonCell$ is surfaced for
confirmation before answering (binding the answer to an explicit question
the user can repudiate); and the certificate records the canonical question
alongside the answer and its grounding, so failures of either kind are
auditable post hoc. Measuring grounding and fidelity rates is empirical
companion work, outside the present theorems.
\end{remark}

\section{Elicitation over cell-decided predicates}\label{sec:elicit}

Cell ambiguity and under-specificity are resolved inside the fixed fiber by
monotone accumulation: admissible facts grow, live cells are eliminated,
until one cell remains or none can be separated. The termination claim
requires the elicitation vocabulary to interact with the partition cleanly,
so that ``indistinguishable'' is a genuine equivalence.

\begin{assumption}[A2$'$: cell-decided separability]\label{a2}
The admissible elicitation vocabulary $\mathcal{E}_c$ consists of decidable
predicates that are \emph{cell-decided}: each $f\in\mathcal{E}_c$ has a
constant truth value on every cell of $P(c)$ (write $f(C)\in\{0,1\}$).
This is natural when $\mathcal{E}_c$ is drawn from the material atoms that
generate the partition. The \emph{observational signature} of a cell is
$\mathrm{sig}(C)=(f(C))_{f\in\mathcal{E}_c}$; $f$ \emph{separates} $C_i,C_j$
iff $f(C_i)\neq f(C_j)$; and $C_i\sim_{\mathcal{E}_c}C_j$ iff
$\mathrm{sig}(C_i)=\mathrm{sig}(C_j)$---an equivalence relation by
construction. A live-cell set satisfies A2$'$ when its cells have pairwise
distinct signatures.
\end{assumption}

\begin{assumption}[A3: responsiveness]\label{a3}
Each reply to an asked $f\in\mathcal{E}_c$ is an admissible ground fact that
decides $f$ (an \emph{informative} reply). Refusals, ambiguous replies, and
non-answers are permitted but do not count against the bound.
\end{assumption}

\begin{assumption}[A4: target consistency]\label{a4}
There is a fixed intended cell $C^\star\in P(c)$ (equivalently, a fixed
intended world $w^\star\in W_c$ with $C^\star$ its cell) such that every
informative reply reports the value of the asked predicate at the target:
$r(f)=f(C^\star)$. A3 alone constrains replies to be admissible and
decisive; it does not make them truthful or mutually consistent, and
soundness below is exactly what A4 adds.
\end{assumption}

\begin{theorem}[Termination and soundness of elicitation]\label{thm:elicit}
Consider any policy that, while
$|\{\mathrm{sig}(C):C\in L_k\}|\geq 2$, selects live cells
$C_i,C_j\in L_k$ with $\mathrm{sig}(C_i)\neq\mathrm{sig}(C_j)$ and some
$f\in\mathcal{E}_c$ with $f(C_i)\neq f(C_j)$ (which exists, since
signatures determine separation), asks $f$, and eliminates the cells whose
signature the reply contradicts. Write $L_k$ for the live set after $k$
informative replies. Then:
\begin{enumerate}[label=(\roman*),nosep]
\item \emph{(Termination, A2$'$--A3.)} The procedure halts after at most
$|P(c)|-1$ informative replies, in a singleton, in a set with a single
shared signature (case (iii)), or in the empty set (case (iv)).
\item \emph{(Soundness, A2$'$--A4.)} $C^\star\in L_k$ for every $k$: an
informative reply eliminates only cells whose signature disagrees with the
reported value $f(C^\star)$, which never includes $C^\star$. Hence under A4
the live set can never become empty; and if the live signatures are
pairwise distinct, the procedure halts in $L_k=\{C^\star\}$---a sound
resolution of the intended governed question.
\item \emph{(Vocabulary limit.)} If the live set's signatures are not
pairwise distinct, the procedure halts as soon as the live set is contained
in a single ambient $\sim_{\mathcal{E}_c}$-class, and reports governed
abstention naming \emph{the ambient equivalence class containing the live
set}. The live set itself need not be a full class---earlier evidence may
already have eliminated some of the class's members---but the ambient class
is a genuine equivalence class, since $\sim_{\mathcal{E}_c}$ is an
equivalence by construction, and the report is the diagnostic that the
elicitation vocabulary, not the corpus, is the binding constraint.
\item \emph{(Inconsistent evidence.)} Without A4, an empty live set is
possible and must be reported as \emph{no compatible cell}: the replies
were individually admissible but jointly inconsistent with every cell (or
untruthful relative to any fixed target). This outcome is distinguished
from the \emph{certified no-match} of the corpus---asserting that no
admissible cell answers the query requires evidence that is truthful and
complete, which A3 alone does not supply. Operationally, no-compatible-cell
routes to human escalation, not to a certified negative.
\item \emph{(Non-responsiveness.)} If A3 fails, no bound on wall-clock
rounds is claimed; the bound counts informative replies only, and
accumulation applies $\cl_c$ to admissible facts, so every intermediate
state is admissible regardless.
\end{enumerate}
\end{theorem}

\begin{proof}
(i) Because each $f$ is cell-decided, an informative reply assigning $f$ its
value contradicts precisely the live cells $C$ with $f(C)$ opposite to the
reply, of which there is at least one when $f$ separates two live cells; the
live count strictly decreases and is finite, so at most $|P(c)|-1$ strict
decreases reach a halting configuration. (ii) By A4 the reported value is
$f(C^\star)$, so $C^\star$ is never among the contradicted cells; induction
gives $C^\star\in L_k$ for all $k$, whence $L_k\neq\emptyset$, and when
signatures are pairwise distinct the halting singleton must be
$\{C^\star\}$. (iii) When all live cells share one signature, no
$f\in\mathcal{E}_c$ separates any pair (signatures determine separation),
the policy's guard fails, and the live set lies in the
$\sim_{\mathcal{E}_c}$-class of that shared signature; containment, not
equality, is claimed. (iv) Without the invariant of (ii), each reply removes
a signature-determined subset and the intersection of the surviving
constraints can be empty; emptiness certifies only that no cell is
consistent with all replies. (v) Non-informative replies leave the live set
unchanged; admissibility is preserved since $\cl_c$ is applied to admissible
inputs and $P(c)$ partitions $W_c$.
\end{proof}

The base/fiber composition of \cite{gas} is unchanged: \emph{which} closure
is in force narrows contravariantly with context refinement; elicitation runs
monotonically inside that closure. Ambiguity about the operative context is
handled at the base (rebinding $c$), ambiguity about the question in the
fiber, and the two are never traded against each other.

\section{Concluding remark}

The central object of this note is the realized governed reuse space
$\Rc = U_{\mathrm{res}}/\!\eqreuse$, reached by the surjection
$U_{\mathrm{res}}\twoheadrightarrow\Rc$: the partition of \emph{resolved}
conversational demands into governed reuse classes, with abstentions held
apart as terminal outcomes of the mode map. Its claims,
stated at their honest strength: the cell normal form (presenting the
governed query $Q_C$, which is an operational object and deliberately not a
partition-semantics question) and the proposition normal form (defined by
exact denotation, hence injective) are single-valued computable maps on
explicit domains (Lemma~\ref{lem:wd}); resolution aggregation is a
conservative design operator whose one non-obvious hazard---closure-induced
cross-cell escape---is characterized exactly by the closure--partition
compatibility condition and detected at the domain check where the condition
fails (Lemma~\ref{lem:cp}); the three equivalences
${\eqread}\subseteq{\eqres}\subseteq{\eqreuse}$ form a refinement chain,
strict under \emph{realized} nondegeneracy conditions---a joint property of
corpus and proposal layer, checkable on deployment logs
(Figure~\ref{fig:quotients})---along which the pipeline's outputs are
invariant, so reuse equivalence is a provable coarsening of reading-level
equivalence, with abstention modes explicitly outside the invariance
(Theorem~\ref{thm:hier}, Remark~\ref{rem:trace}), and with query-key reuse
distinguished from particular-answer reuse and $\eqreuse$ identified as the
kernel of the resolution map---so the quotient measures realized demand, not
a relabeling of the answer partition
(Corollary~\ref{cor:factor}, Remark~\ref{rem:whatreused}); the whole map is decidable conditional
on a computable proposal layer---deliberately immediate, with the content in
the placement of the assumptions---with polynomial checks in the ground and
fixed-program regimes only over bit-vector representations
(Theorem~\ref{thm:dec}, Proposition~\ref{prop:cx}); and of the three
orthogonal safety properties---policy admissibility, factual grounding,
intent fidelity---the closure proves exactly the first, for arbitrary
proposers, with certificate possession derived from the proof-producing
closure rather than asserted
(Proposition~\ref{prop:auth}, Remark~\ref{rem:fidelity}). Elicitation over
cell-decided predicates terminates, is sound under target consistency---the
intended cell is an invariant of the live set---and reports its two failure
modes honestly: a vocabulary limit as containment in an ambient
observational equivalence class, and inconsistent evidence as no compatible
cell, distinct from a certified no-match (Theorem~\ref{thm:elicit}). The
order theory underneath is
classical \cite{tarski}; the contribution is the characterization of when
two different conversational demands are legitimately the same reusable
governed object---the question on which certified-reuse economics ultimately
rests.

\paragraph{Disclosure} The authors are affiliated with Minerva CQ, which has
commercial interests in AI-governance tooling that builds on these results.


\begin{thebibliography}{9}\small
\bibitem{gas} C.~Spera, R.~Garcia, Certified Resolution: a formal theory of
governed answer spaces for enterprise AI, Minerva CQ working paper (2026).
\bibitem{gro} J.~Groenendijk, M.~Stokhof, Studies on the Semantics of
Questions and the Pragmatics of Answers, PhD thesis, Univ.\ of Amsterdam
(1984).
\bibitem{belnap} N.D.~Belnap, T.B.~Steel, The Logic of Questions and Answers,
Yale University Press, 1976.
\bibitem{dowling} W.F.~Dowling, J.H.~Gallier, Linear-time algorithms for
testing the satisfiability of propositional Horn formulae, J.\ Logic
Programming 1 (1984) 267--284.
\bibitem{abiteboul} S.~Abiteboul, R.~Hull, V.~Vianu, Foundations of
Databases, Addison-Wesley, 1995.
\bibitem{capdatalog} C.~Spera, Capability safety as Datalog: a foundational
equivalence, arXiv:2603.26725 (2026).
\bibitem{gptcache} F.~Bang, GPTCache: an open-source semantic cache for LLM
applications, in: Proc.\ EMNLP Industry Track, 2023, pp.\ 212--218.
\bibitem{tarski} A.~Tarski, A lattice-theoretical fixpoint theorem and its
applications, Pacific J.\ Math.\ 5 (1955) 285--309.
\end{thebibliography}
\end{document}